\documentclass[letterpaper, 10 pt, conference]{ieeeconf}  

\IEEEoverridecommandlockouts                              

\usepackage{fancyhdr}
\usepackage{graphicx}
\usepackage{amssymb}
\usepackage{epstopdf}
\usepackage{amsmath} 	
\usepackage{amssymb}
\usepackage{cite}
\usepackage{multirow}
\usepackage{wrapfig}
\usepackage{lastpage}
\usepackage{subcaption}

\usepackage{algpseudocode}
\usepackage{algorithm}
\usepackage[english]{babel}

\usepackage{amsthm}
\usepackage{booktabs}

\usepackage[dvipsnames]{xcolor}

\definecolor{myorange}{HTML}{ffa600}
\definecolor{myblue}{HTML}{003d5c}

\theoremstyle{plain}
\newtheorem{theorem}{Theorem}               

\theoremstyle{definition}
\newtheorem{definition}{Definition}         
\newtheorem{assumption}{Assumption}         

\title{\LARGE \bf
Jacobian State Augmented Viable Kinematic Safety Filter
}

\author{Siddharth Singh$^{1}$ and Joshua Vaughan$^{1}$
\thanks{Notice: This manuscript has been authored by UT-Battelle, LLC, under contract DE-AC05-00OR22725 with the US Department of Energy (DOE). The US government retains and the publisher, by accepting the article for publication, acknowledges that the US government retains a nonexclusive, paid-up, irrevocable, worldwide license to publish or reproduce the published form of this manuscript, or allow others to do so, for US government purposes. DOE will provide public access to these results of federally sponsored research in accordance with the DOE Public Access Plan ( https://www.energy.gov/doe-public-access-plan )}
\thanks{$^{1}$Siddharth Singh and Joshua Vaughan are in the Manufacturing Robotics and Controls Group within the Manufacturing Science Division, Energy Science \& Technology Directorate of the Oak Ridge National Lab.
        Email: {\tt\small singhsk@ornl.gov} and {\tt\small vaughanje@ornl.gov}}%
}

\usepackage{xargs}                      
\usepackage[pdftex,dvipsnames]{xcolor}  
\usepackage[colorinlistoftodos,prependcaption,textsize=tiny,linecolor=black,backgroundcolor=lightgray]{todonotes}
\newcommandx{\unsure}[2][1=]{\todo[linecolor=red,backgroundcolor=red!25,bordercolor=red,#1]{#2}}
\newcommandx{\change}[2][1=]{\todo[linecolor=blue,backgroundcolor=blue!25,bordercolor=blue,#1]{#2}}
\newcommandx{\info}[2][1=]{\todo[linecolor=black,backgroundcolor=lightgray,#1]{#2}}
\newcommandx{\improvement}[2][1=]{\todo[linecolor=OliveGreen,backgroundcolor=OliveGreen!25,bordercolor=OliveGreen,#1]{#2}}
\newcommandx{\hide}[2][1=]{\todo[disable,#1]{#2}}
\newcommandx{\red}[2][1=]{\textcolor{red}{#2}}

\begin{document}

\maketitle
\thispagestyle{empty}
\pagestyle{empty}

\begin{abstract}
    Safety Filters (SF) are leveraged to design certifiably safe controllers, which ensures forward invariance on a defined safe set when the filter is feasible. Predictive Safety Filters (PSF) embed these constraints to a closed-loop model-based optimization method to improve constraint satisfaction, and reduce feasibility loss caused by myopic, one-step filtering. Related viability-based methods emphasize approximating a viability kernel which ensures admissible control for all future time, but requires significant computational expense. This work introduces a Jacobian-based kinematic safety filter leveraging an augmented state variable with task-space velocity and filters task-space acceleration to preserve the feasibility of a low-level filter while enforcing joint position and velocity. The proposed method introduces a nested Quadratic Program (QP) formulation, defines a conservative inner approximation of viability envelope and computes a viable planner online. The method is validated in a simplified 2R planar manipulator simulation environment and shows significant improvement compared to a baseline Control Barrier Function Quadratic Program (CBF-QP) based filter. The experiments highlights the improvements in terms of computation time compared to traditional PSF and Reference Governor approach while also showing significant reduction in feasibility violations. Lastly, a simplified demonstration on a simulated 6-DoF manipulator highlights the potential scalability of the method to higher degree-of-freedom setups.
\end{abstract}

\section{Introduction}
Many settings employ manipulators that are commanded at the kinematic level through task-space velocity references or Jacobian-based inverse kinematics rather than full rigid-body dynamics. In such settings, safety mechanisms should remain lightweight, online implementable, and compatible with limited model knowledge. 
Safe motion planning and controls in robots is critical, and has seen substantial developments in recent decades\cite{annurev_safetfilters_haimin}. This has included development of reachability-based methods as well as safety filters. While there are several model-based and data-driven approaches, there has been limited attention on kinematically controlled robots\cite{review_datadriven_sf}. In several cases,  such as manufacturing, where robots are often controlled kinematically, complete model knowledge is often not available. Hence, most such methods rely on pre-programmed trajectories and sampling based planners, which only take the path planning into account, overlooking the input constraints\cite{lavalle}. This emphasizes the need to develop model-light safety filtering approaches for kinematically-driven systems which can improve operational efficiency. 

Safety filtering ensures that the controlled plant does not violate pre-defined state and input constraints. When utilizing a safety filter at low-level, a one step input can push the robot trajectory such that the Quadratic Program (QP) becomes infeasible implying that even if a safe controller exists, this can lead to an infeasible problem with no solution\cite{review_datadriven_sf}. Predictive Safety Filters (PSF) tackle this problem by solving a predictive control problem \cite{annurev_safetfilters_haimin, review_datadriven_sf}. By embedding safety constraints directly into the optimization, the filter defines a safety-recovering control law. This ensures that the predicted state trajectory is forward invariant to the defined safe set. Additionally, recent developments have formulated this challenge as a viability-based problem\cite{zhang2025viability}. Such methods take the control input constraints and the closed-loop model into account during controller synthesis to define a viable set. PSF-based methods often utilize complete model definition and solve a Non-linear Model Predictive Control (NMPC) problem. This does not just rely on complete model knowledge but is also computationally expensive. Similar issues occur for viability preserving methods. While Control-Barrier-Function (CBF)-based approaches can ensure safety by ensuring forward invariance, this can lead to oscillatory behavior near the safe set boundary and limit the task performance \cite{Ames_cbf}.


This work proposes a viability preserving method for kinematic safety filters leveraging the Jacobian. The proposed method aims to ensure that the low-level kinematic safety filter is always viable by utilizing a high-level planner. While the proposed method does not explicitly provide kinodynamic safety, i.e. ensuring that the joint torques are within a defined safe set, the approach does ensure task space acceleration admissibility (as a proxy to kinodynamic constraints) while alleviating the need of complete model knowledge. Additionally, the proposed method can be leveraged in cases where the reference trajectory is either provided as the velocity in the task space or as a reference acceleration trajectory in the task space. The main contributions of the work are:

\begin{itemize}
    \item A nested-QP based model-light approach which ensures a viable high-level plan and kinematic safety of the low-level execution, resulting in reduction of constraint violation. 
    \item An ideal scenario Maximal Control Invariant Set (MCIS) formulation with a simplified viability envelop definition to handle tractability.
\end{itemize}
The method is validated in simulation for a 2-DoF Rotational (2R) Planar Manipulator. Both qualitative and quantitative results highlight the characteristics and online realization of the proposed approach against standard CBF approach as well as reference governor and PSF-based approaches. An additional 6-DoF-manipulator-based simulation highlights the method's potential scalability for higher dimensional setups. 

The next section discusses the past work relevant to this research and highlights their respective limitations for the problem of interest. Following this, the preliminary background is provided in Sec.~\ref{sec:prelim}, and Sec.~\ref{sec:proposed} entails the proposed method. Experimental validation and observations are presented in Sec.~\ref{sec:exp}, and finally the conclusions are drawn in Sec.~\ref{sec:conc}. 

\section{Related Work}\label{sec:related}
The past decade has seen substantial development in safe and robust control and planning, especially in robotics \cite{annurev_safetfilters_haimin}. Broadly, three major classes of methods have emerged: Hamilton-Jacobi-(HJ)-reachability based approaches, Control-Barrier-Function (CBF)-based methods, and Predictive Safety Filters (PSFs)\cite{review_datadriven_sf}. HJ-reachability-based methods provide strong safety guarantees through reachable set and invariant set computations, and have also motivated barrier-value and liveness filtering formulations \cite{choi2021robust,choi2023safety}. However, these methods often face significant scalability challenges for higher-dimensional systems such as robotic manipulators pertaining to the large state space.

Predictive Safety Filters and related receding-horizon safety methods explicitly address future constraint satisfaction by embedding safety into a predictive optimization problem \cite{wabersich2021predictive,leeman2023predictive,leeman2024stability}. These approaches are attractive because they directly tackle recursive feasibility and reduce the myopic behavior of one-step filters and can be combined with both HJB and CBF-based methods \cite{choi2023safety, psf_cbf}. However, they generally require an explicit predictive model and incur substantially higher computational cost, especially when implemented as nonlinear model predictive control problems. This becomes restrictive in settings where the robot is commanded through task-space kinematic references, such as end-effector velocity commands or demonstration-derived trajectories, and where full dynamic model information may not be readily available.

In contrast, CBF-based quadratic-program filters have proven effective for real-time safety-critical control due to their relative computational efficiency and ease of online implementation \cite{Ames_cbf}. Kinematic energy-based CBF formulations have further extended these ideas to robotic systems commanded at the Jacobian level \cite{kineamtic_cbf, energy_cbf}. Despite these advantages, standard one-step CBF-QP filters can still suffer from future infeasibility; an aggressive nominal command may satisfy the current-step safety constraint while pushing the state to a region where the next optimization becomes infeasible. Combining CBFs with predictive layers can alleviate this issue, but often reintroduces moving-horizon optimization and greater model dependence \cite{psf_cbf, psf_cbf2} or require significant data collection for learning based controllers\cite{psf_cbf3}.

Recent work has also explored safety filtering through viability-preserving and governor-style formulations\cite{ubc2012viability, choi2021robust}. Viability-based methods aim to characterize sets from which an admissible safe input exists for all future time, but exact set computation can be difficult in high-dimensional robotic systems \cite{zhang2025viability, roemer2026path}. Likewise, reference- and command-governor approaches shape the commanded reference to maintain constraint satisfaction without replacing the nominal controller entirely \cite{gilbert1985fast, governor_kod, li2023governor}. 
Among prior works, kinematic CBF control for robotic systems \cite{kineamtic_cbf} is the closest low-level baseline to the setting of interest in this work, while predictive safety filters \cite{wabersich2021predictive,leeman2023predictive} and governor-style safe tracking methods \cite{li2023governor} are the closest approaches addressing future admissibility more directly.

Rather than solving a full predictive dynamic optimization problem, the proposed method uses a Jacobian-based augmented-state formulation and a nested QP architecture to regulate task-space acceleration so that the induced joint-space command remains within a conservative viability envelope. This preserves the feasibility of the low-level kinematic CBF-QP while remaining lightweight enough for online implementation.

\section{Preliminaries}\label{sec:prelim}
Given a desired kinematic trajectory of the end-effector in the task space, $\dot{x}_{des} \in \mathbb{R}^m$, the desired nominal joint space velocity $\dot{q}_{nom} \in \mathbb{R}^n$ at any given time can be computed as:
\begin{equation}\label{eq:jacobian_def}
    \dot{q}_{nom} = J^{\dagger}(q)\dot{x}_{des}
\end{equation}
where $n$ is the Degree-of-Freedom (DoF) of the robot and $J^{\dagger}(q)\in \mathbb{R}^{n\times m}$ is the Pseudo Inverse of the Jacobian ($J \in \mathbb{R}^{m \times n}$) at the joint position $q \in \mathbb{R}^n$. In case when $J$ is not full-row rank, the pseudo inverse is replaced by the damped pseudo inverse.

To ensure safety, a Control Barrier Function Quadratic Program (CBF-QP)-based filter is employed to ensure safe certifiable operation. The CBF over the safe set of joint positions is defined as:

\begin{definition}[Valid CBF]\label{def:valid_cbf}
The scalar valued continuously differentiable function $h :~\mathbb{R}^n \rightarrow \mathbb{R}$ is a valid CBF for the joint velocity system if $\text{sup}_{\dot q \in\mathcal{V}} \nabla_q h(q)\dot q \geq -\alpha h(q)$ holds true for all $q~\in~\mathcal{Q}_{safe}$ where $\mathcal{Q}_{safe} := \{q ~\in~ \mathbb{R}^n ~|~ h(q) > 0\}$.
\end{definition}

Here $\alpha > 0 \in \mathbb{R}$, a positive constant, is the decay constant of the CBF and $\mathcal{V} \subset \mathbb{R}^n$ is the set of feasible joint velocities. Utilizing this, a QP filter can be defined as:
\begin{equation}\label{eq:basic_cbf}
\begin{aligned}  
    \dot{q}^* :=
    \min_{\dot q ~\in \mathcal{V}} \quad & \left\|\dot{q} - \dot{q}_{\text{nom}}\right\|^{2} \\
    \textrm{s.t.} \quad & \nabla_q h(q)^\top \dot{q}\geq -\alpha h(q) \\
    \quad & \dot{q} \in \mathcal{V} := \{\dot{q} \in \mathbb{R}^n ~| ~\dot{q}_{min} \leq \dot{q} \leq \dot{q}_{max}\}  \\
\end{aligned}
\end{equation}

While the QP in \eqref{eq:basic_cbf} can ensure the safety of the joint velocity trajectory, in some cases the desired nominal velocity command $\dot{q}_{nom}$ is so large that the future state is pushed to an unsafe region. Additionally, another possibility arises where the controller is theoretically safe but commanded velocity is so high that it renders the QP infeasible at the next step. In such cases, a fallback controller which uses the closest boundary value as $\dot{q}^*$ within the input feasibility constraints is often invoked in practice.

\section{Proposed Method}\label{sec:proposed}
This problem can be tackled by defining an augmented state variable which leverages the knowledge of the kinematic constraints and yields a viable governor for high-level planner such that the feasibility of the low-level safety controller is preserved. 
Figure~\ref{fig:main} shows the schematic of the proposed approach with the bi-level nested QP which allows defining the safety constraint sets at different levels.

By leveraging the augmented state definition and the Jacobian, the task space and the joint space velocity constraints are embedded within the planner dynamics. Additionally, the proposed approach allows embedding the joint space and task space position constraints directly. The task space position constraints can only be embedded indirectly. This results in a viable planner which preserves the forward invariance of the low-level controller by maintaining feasibility.

\begin{figure}
    \centering
    \includegraphics[width=\linewidth]{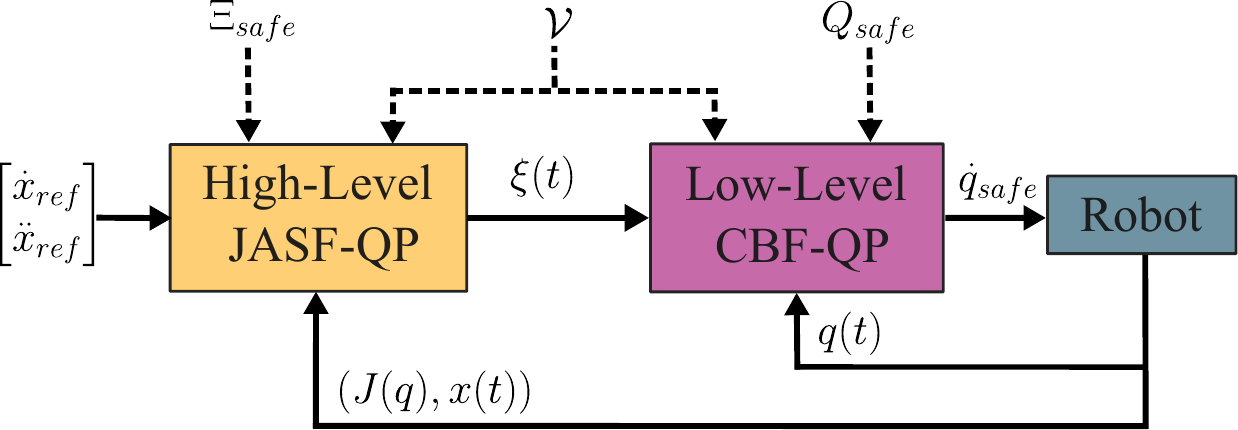}
    \caption{Schematic representation of the proposed method with the two nested QPs.}
    \label{fig:main}
\end{figure}

\subsection{Augmented State Variable Definition}
Ensuring the feasibility of low-level controller requires that the commanded velocity be viable. 
To handle the tracking problem, a high-level planner leverages an augmented variable defined as $z = [q, \dot{x}]^\top \in \mathbb{R}^{n+m}$. The task space acceleration $\ddot{x}$ is treated as the input to the augmented state dynamics, and the applied input $u_{safe} = \ddot{x}_{safe}$ is filtered from the reference $u_{ref} = \ddot{x}_{ref}$ by the outer QP. Here $\ddot{x}_{safe}$ is the filtered cartesian acceleration after ensuring feasibility of the inner CBF-QP. To ensure this, the induced nominal velocity is computed, $\xi(q,\dot{x}) = J^\dagger(q)\dot{x} \in \mathbb{R}^n$. The induced nominal velocity is provided to the inner CBF-QP as the nominal command. 
Consequently, the inner CBF-QP can be reformulated as: 
\begin{equation}\label{eq:inner_qp}
\begin{aligned}  
    \kappa(q,\xi(q,\dot{x})) :=
    \min_{\dot q ~\in \mathcal{V}} \quad & \left\|\dot{q} - \xi\right\|^{2} \\
    \textrm{s.t.} \quad & \nabla_q h(q)^\top \dot{q}\geq -\alpha (h(q)) \\
     & \dot{q} \in \mathcal{V}\\
\end{aligned}
\end{equation}

Additionally, the ``feasibility" of the inner CBF-QP is defined using a feasibility indicator:

\begin{definition}[Feasibility Indicator]\label{def:feas_ind}
Given the feasible joint velocity set $\mathcal{V}$, the feasibility indicator is defined as:
\begin{equation}\label{eq:feas}
    \text{Feas}(q) = \mathbf{1} \bigg[ ~\exists~ \dot q ~\in~ \mathcal{V} ~|~ \nabla_q h(q)^\top\dot q\geq -\alpha h(q) \bigg]
\end{equation}
\end{definition}

Subsequently, the first-order plant dynamics of the augmented state variable are given by:

\begin{equation}\label{eq:zdot}
    \dot{z} =  \begin{bmatrix}
                \dot{q} \\
                \ddot{x}
                \end{bmatrix}
        = \begin{bmatrix}
            \kappa(q,\xi(q,\dot{x})) \\
            u
            \end{bmatrix}
\end{equation}
where $u$ is the cartesian acceleration from the planner, and $\mathcal{U} \subset \mathbb{R}^m$ is the bounded set on cartesian acceleration which are defined by the task acceleration limits. To enforce these constraints while supporting feasibility of the inner CBF-QP, the viability envelope is first defined, after which an outer QP is introduced.

\subsection{Ideal Feasible Set for Planner}\label{sec:mcis}
To ensure that the low-level controller is admissible, the inner CBF-QP safety filter should always be feasible. Ideally, this is ensured by the Maximal Control Invariant Set (MCIS) of the augmented state variable $z$. Before defining the MCIS, the feasible constraint set on $z$ is defined as:

\begin{equation}
    \mathcal{C} := (\mathcal{D}_q ~\cap~\mathcal{Q}_{safe}) \times \mathcal{V}_{x}
\end{equation}
where $\mathcal{D}_q$ a domain on joint values where $J$ is smooth and away from singularities. More practical relaxations on $J$ are introduced in the later subsections. Additionally, $\mathcal{V}_{x} \in \mathbb{R}^m$ is the cartesian velocity constraints set. Given this feasible set the MCIS can be defined as: 

\begin{definition}[MCIS on $z$]\label{def:mcis}
    Given $z := [q, \dot{x}]^\top$, the state dynamics $\dot{z}$ as defined in \eqref{eq:zdot}, and the input constraints $u \in \mathcal{U}$, the MCIS $\mathcal{C}^* \subset \mathcal{C}$ on $z$ is defined as the set such that starting in the $\mathcal{C}$, there exists an input that keeps the state forward invariant in $\mathcal{C}$ for all time, i.e.
    \begin{equation}
        \mathcal{C}^* = \{z_0 \in \mathcal{C} ~\mid~ \exists~ u(\cdot) \in \mathcal{U}~s.t. ~z(t) \in \mathcal{C}, ~\forall t \geq 0\}
    \end{equation}
\end{definition}

Leveraging the MCIS, it can ensured that the there is always a commanded input which keeps the inner CBF-QP feasible at all times. 
\\
\textbf{Projection onto $\xi$-space}: While the MCIS is defined on the $z$-space, it can also be projected on to the $\xi$-space (i.e. the space of the internal joint space velocity) to introduce $\Xi^* := \text{Proj}_{\xi}(\mathcal{C}^*)$ where the $\text{Proj}_{\xi}(\cdot)$ is the projection operator that leverages the Jacobian $J$ (since $\xi = J^{\dagger}(q)\dot{x}$) to project $\mathcal{C}^*$ on to the $\xi$-space. While it is possible to estimate the MCIS for lower DoF setups (such as n = 2 or 3), for higher DoF setups it remains intractable.
To tackle this issue, an approximate viability envelope is defined in the following subsection. 

\subsection{Viability Envelope and Outer Safety Filter QP}\label{sec:jasf_qp}
Since the commanded velocity is the only variable effecting the low-level controller, the viability envelope is defined as:

\begin{equation}
    \Xi_{safe} := \{\xi \in \mathbb{R}^n ~\mid~ \xi_{min} \leq \xi \leq \xi_{max}\}
\end{equation}
It must be noted that $\Xi_{safe} \subseteq \Xi^*$ and can be calibrated offline. For higher DoF systems where $\Xi^*$ computation is intractable, a conservative $\Xi_{safe}$ can be chosen heuristically and validated empirically.    

Leveraging this definition of the viability envelope, the outer QP safety filter is formulated; to prevent confusion, it is referred to as JASF-QP (Jacobian state Augmented Safety Filter Quadratic Program) henceforth. To define the safety filter on the viability envelope, the set constraints are defined as linear inequalities instead of a non-smooth barrier functions. The upper and lower bounds are translated as:

\begin{equation}
\begin{aligned}
    \xi_j \leq\xi_{max,j} \Longleftrightarrow e_j^\top\xi \leq \xi_{max,j} \\
    \xi_j \geq\xi_{min,j} \Longleftrightarrow -e_j^\top\xi \leq -\xi_{min,j} 
\end{aligned}
\end{equation}
where $e_j$ are the standard basis vector choosing the $j^{th}$ constraint. This allows the following formulation as $a_i^\top\xi \leq b_i$ for $i = [1, \cdots, N]$, where $N = 2n$. This allows one to introduce a barrier function $g_i(\xi) := b_i - a_i^\top\xi \geq 0$. Consequently, to define the CBF style constraint: 

\begin{equation}
    \begin{aligned}
        \dot{g}_i(\xi) \geq -\gamma~g_i(\xi) \Longleftrightarrow -a_i^\top\dot{\xi} \geq -\gamma(b_i - a_i^\top\xi)\\  \Longleftrightarrow a_i^\top\dot{\xi} \leq \gamma(b_i - a_i^\top\xi)
    \end{aligned}
\end{equation}where $\gamma > 0$ is a scalar value that controls the decay rate of the inequalities. With a concatenation, the vector form can also be presented as:

\begin{equation}
    G(\xi) := \begin{bmatrix}
        g_1(\xi) \\
        \vdots\\
        g_N(\xi)
    \end{bmatrix} = b - A\xi
\end{equation}
where $A \in \mathbb{R}^{N\times n}$ has rows $a_i^\top$, and $b\in \mathbb{R}^N$ stacks the $b_i$. This allows one to redefine the viability set as: 
\begin{equation}
    \Xi_{safe} := \{\xi \in \mathbb{R}^n ~\mid~ G(\xi)\geq 0\}
\end{equation}
Given the concatenated constraints defining the viability envelope the JASF-QP can be formulated as:
\begin{equation}\label{eq:jasf_qp}
\begin{gathered}
u_{\mathrm{safe}} := \min_{u \in \mathcal{U}} \quad \|u - u_{\mathrm{ref}}\|^2 \\
\text{s.t.}\quad a_i^\top\left(\dot{J}^\dagger(q,\dot{q})\dot{x} + J^\dagger(q)u\right)
\le \gamma(b_i - a_i^\top \xi) \quad \forall i
\end{gathered}
\end{equation}
It is to be noted here that $\dot q$ here is the current joint velocity of the closed loop system, making the constraint affine in $u$.
Before the proof of forward invariance is provided the following assumptions are made:

\begin{assumption}\label{assump:jacobian}
It is assumed that $J^{\dagger}(q)$ is $\mathcal{C}^1$ on a domain $\mathcal{D}_q \in \mathbb{R}^n$ (i.e., away from singularities) and that $\dot{J}^\dagger(q,\dot{q})$ exists and is is locally bounded for all $(q,\dot q)$ encountered.  
\end{assumption}

\begin{assumption}\label{assump:inner_qp}
The inner CBF-QP \eqref{eq:inner_qp} is feasible on the considered region, and the optimizer map $\kappa(q, \xi(q,\dot{x}))$ is bounded, single valued and locally Lipschitz (or piece-wise Lipschitz).   
\end{assumption}

\begin{assumption}\label{assump:jas_qp_feas}
It is assumed that outer JASF-QP \eqref{eq:jasf_qp} is feasible on the input constraints set $\mathcal{U}$.
\end{assumption}

\textbf{Remark}: Both Assumption~\ref{assump:jacobian} and Assumption~\ref{assump:jas_qp_feas} can be relaxed in practice by introducing slack variables, for computing the Pseudo Inverse of $J$ and in the optimization cost function in \eqref{eq:jasf_qp}. 

Following this, the forward invariance of the viability envelope can be guaranteed by:

\begin{theorem}\label{theo:viab_env}
    Following the Assumption~\ref{assump:jacobian} and Assumption~\ref{assump:jas_qp_feas}, if $\xi(q(0), \dot{x}(0)) \in \Xi_{safe}$ and the outer QP is feasible at all times, then $\xi(q(t),\dot{x}(t)) \in \Xi_{safe}$ for all $t \geq 0$.
\end{theorem}

\begin{proof}
For any each constraint $g_i(\xi(t)) = b_i - a_i^\top\xi(t)$, the QP yields in an input $u$ that satisfies the constraint which is equivalent to 

\[
a_i^\top\left(\dot{J}^\dagger(q,\dot{q})\dot{x} + J^\dagger(q)u\right) \leq \gamma(b_i - a_i^\top\xi)
\]

or 

\[
-a_i^\top \dot{\xi}(t) \geq -\gamma(b_i - a_i^\top\xi(t)) \Longleftrightarrow \dot{g}_i(\xi(t)) \geq -\gamma g_i(\xi(t))
\]
since $\dot{g}_i(\xi(t)) = -a_i^\top\dot{\xi}(t)$. This holds for every $t \geq 0$ for every $i$, owing to feasibility of the JASF-QP at all times. It can be shown as a result of the standard comparison lemma that since:
\begin{equation}
    \begin{aligned}
        \dot{g}_i(\xi(t)) \geq -\gamma g_i(\xi(t)) \\
        \implies g_i(\xi(t)) \geq e^{-\gamma t}g_i(\xi(0)) \geq 0        
    \end{aligned}
\end{equation}
\end{proof}

\textbf{Remark}: This is a standard exponential CBF sufficiency condition for control invariance in a polytope (intersection of half spaces). It establishes forward invariance of $\Xi_{safe}$ by construction. It is to be noted that this does not require the computation of $\Xi^*$. 

A direct consequence of Theorem~\ref{theo:viab_env} is that the inner QP's correction $\|\dot{q}^*(t) - \xi(t)\|$ is bounded by $\|\xi\|$, since $\xi(t)$ remains within the calibrated $\Xi_{safe} \subset \mathcal{V}$. This also explains the substantial reduction in max $\xi$ observed in empirical experiments seen in Sec.~\ref{sec:exp_2r_manip}.

By leveraging the proposed nested QP formulation, a viability envelope can prevent the saturation of low-level controller improving command preservation by bounded tracking error of the inner filter. 

\subsection{Approximate Discrete Time Formulation}
While the both the safety filters can be formulated for discrete time systems, for a very small $\Delta t$, the following finite difference approximations can also be used for realization of the proposed method. 
The task space velocity can be approximated as $\dot{x}_{k+1} = v_{k+1} = v_k + u_k \Delta t$. Similarly, the induced joint nominal velocity is computed as $\xi_{k+1} = J^\dagger(q_k)v_{k+1}$, and the solution of the resulting inner QP can be represented as $\dot{q}_k = \kappa(q_k, \xi_{k+1})$. The joint positions are obtained by forward integration, $q_{k+1} = q_k + \dot{q}_k\Delta t$. Additionally, since $\dot{J}^\dagger$ can be computed with a lagged value of $\dot{q}_{k-1}$ as:

\begin{equation}
    \dot{J}^\dagger \approx \frac{J^\dagger(q_k) - J^\dagger(q_{k-1})}{\Delta t}
\end{equation}

Consequently, Algorithm~\ref{alg:jasf-cbf} shows the implementation of the proposed method with the discrete time approximations. The algorithm utilizes a specific goal tending approach to compute the $\dot x_{ref}$ and $\ddot{x}_{ref}$ as:

\begin{equation}\label{eq:pd_crtlr}
    \begin{aligned}
        \dot{x}_{ref} = K_p e_x, \quad e_x = x_{goal} - x_{k} \\
        \ddot{x}_{ref} = K_p e_x + K_d\dot{e}_x, \quad \dot{e}_x = \dot{x}_{goal} - \dot{x}_{k}
    \end{aligned}
\end{equation}
where $K_p$ and $K_d$ are user defined parameters. The same algorithm can be used when a reference is provided from sources such as human demonstrations or a trajectory generator by replacing the computed values from stored values. 

\begin{algorithm}[ht]
\caption{Online Two-Level Safety Control (JASF Outer, CBF-QP Inner)}
\label{alg:jasf-cbf}
\begin{algorithmic}[1]
\Require $x_{\mathrm{goal}}$, sampling period $\Delta t$, bounds $\Xi_{safe} := \{\xi_{\min},\xi_{\max}\}$, joint limits $ \mathcal{V} := \{\dot{q}_{\min},\dot{q}_{\max}\}$
\State Initialize $q_0,\ \dot{x}_{\mathrm{safe},0}$ and set $k\gets 0$
\While{not at goal}

    \State $x_k \gets f_{\mathrm{FK}}(q_k)$ \Comment{$f_{\mathrm{FK}}$: Forward Kinematics}
    \State $\dot{x}_k \gets J(q_k)\,\dot{q}_k$ \Comment{or measured/estimated}

    \State $(\dot{x}_{\mathrm{ref},k}, \ddot{x}_{\mathrm{ref},k}) \gets \textsc{Planner}(x_{\mathrm{goal}}, x_k, \dot{x}_k)$

    \State $\xi_k \gets J^{\dagger}(q_k)\,\dot{x}_{\mathrm{safe},k}$


    \State $u_{\mathrm{safe},k} \gets \textsc{JASF-QP}\!\left(\ddot{x}_{\mathrm{ref},k}, q_k, \xi_k, \dot{x}_{\mathrm{safe},k}, \right.$
    \Statex \quad \quad \quad \quad \quad \quad \quad \quad \quad \quad \quad \quad $\left. \dot{q}_k, J(q_k), \Xi_{safe}\right)$

    \State $\dot{x}_{\mathrm{safe},k+1} \gets \dot{x}_{\mathrm{safe},k} + u_{\mathrm{safe},k}\,\Delta t$
    \State $\xi_{k+1} \gets J^{\dagger}(q_k)\,\dot{x}_{\mathrm{safe},k+1}$

    \State $\dot{q}_{\mathrm{safe},k} \gets \textsc{CBF-QP}\!\left(\xi_{k+1}, q_k, h(q_k), \frac{\partial h}{\partial q}(q_k), \mathcal{V}\right)$

    \State $q_{k+1} \gets q_k + \dot{q}_{\mathrm{safe},k}\,\Delta t$

    \State $k \gets k+1$
\EndWhile
\end{algorithmic}
\end{algorithm}

\section{Experiments}\label{sec:exp}
The proposed method is validated in simulation on a planar 2R manipulator and a circular obstacle within the robot's workspace. Firstly, the proposed method is compared against baseline, pure CBF based method. These methods employ a simple box constraints approach for defining the Viability Envelop (VE). The method is also compared against a reference governor approach and NMPC-PSF approach. The computational performance with respect to NMPC-PSF is studied over different horizon lengths. Other characteristics such as effect of the decay factor $\gamma$ on behavior of the JASF-QP is also analyzed. Lastly, to validate the scalability of the method, it is implemented on a 6-DoF robotic manipulator in simulation for goal-reaching and obstacle avoidance scenario with the viability envelop. All the experiments were conducted using Python and were implemented on a Intel\textregistered~ Core Ultra 7 CPU with 14 Cores, 1 TB SSD, and 32 GB RAM.
\begin{table}[htbp]
    \centering
    \caption{Design Parameters}
        \begin{tabular}{l|c|c}
        \hline
        \multicolumn{2}{c|}{\textbf{Parameter}} & \textbf{Value} \\ \hline
        Gains & [$K_p$, $K_d$] & [1.5, 1.0] \\ 
        Link Lengths & $[l_1, l_2]$ & [0.4, 0.4] (m)\\ 
        Decay rates & $[\gamma, \alpha]$ & [2.0, 1.0] \\ 
        Max Time & $T$ & 12 (s) \\ \hline
        \end{tabular}
    \label{tab:params}
\end{table}


\subsection{Analysis on Planar Manipulator}\label{sec:exp_2r_manip}
The method is validated in simulation for a planar 2-DoF (2R) manipulator. The aim is to ensure that the end-effector reaches the goal while avoiding collision with an obstacle. The reference acceleration $\ddot{x}_{ref}$ is computed with a Proportional Derivative planner as defined in \eqref{eq:pd_crtlr}, where $x_{goal}$ and $x_{start}$ are the desired goal position and start position for robot tip respectively. Table~\ref{tab:params} lists all the parameters. Here, $K_p$ and $K_d$ are the gains in \eqref{eq:pd_crtlr}, $[l_1, l_2]$ are the link lengths, $T$ is the total duration of the trajectory, and $(\gamma, \alpha)$ are the higher-level and lower-level CBF decay rates as discussed in Sec.~\ref{sec:jasf_qp} and Sec.~\ref{sec:prelim} respectively.

\subsubsection{Qualitative Analysis}
Firstly, the proposed method (JASF) is validated qualitatively against the baseline method (vanilla CBF). Figure~\ref{fig:qual} compares the a) end-effector trajectory, b) $h(q)$, which is essentially a measure of distance from obstacle, and c) the evolution of the joint trajectories. It can be seen in Fig.~\ref{fig:sub1}, that the trajectory with the JASF is less erratic. It also reduces the number of violations of the VE as seen in Fig.~\ref{fig:sub3}, and as seen in Fig.~\ref{fig:sub2}, the proposed method also minimizes the CBF constraint of the low-level gets activated much sooner. Figure~\ref{fig:qual_sweep} illustrates the comparison between the trajectories for 6 random cases where the initial pose of the arm and the final goal of the tip are chosen randomly. As seen, the proposed approach reduces the trajectory length and erratic motion planning for several cases.
\begin{figure}[htbp]
    \centering
    \begin{subfigure}[b]{0.45\linewidth}
        \centering
        \includegraphics[width=\textwidth]{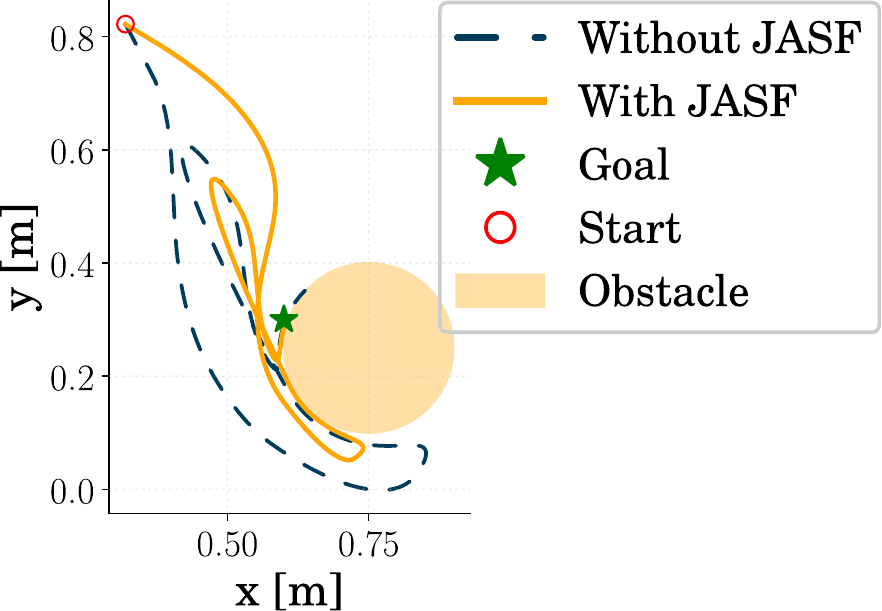}
        \caption{End-effector trajectory}
        \label{fig:sub1}
    \end{subfigure}
    \hfill 
    \begin{subfigure}[b]{0.45\linewidth}
        \centering
        \includegraphics[width=\textwidth]{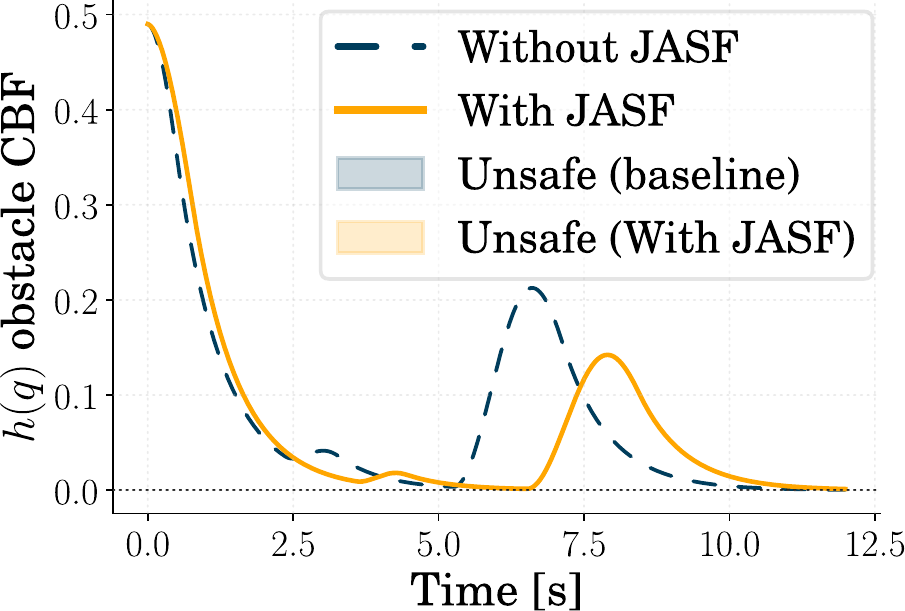}
        \caption{CBF value $h(q)$}
        \label{fig:sub2}
    \end{subfigure}
    
    \vspace{0.5cm} 

    \begin{subfigure}[b]{0.45\textwidth}
        \centering
        \includegraphics[width=\textwidth]{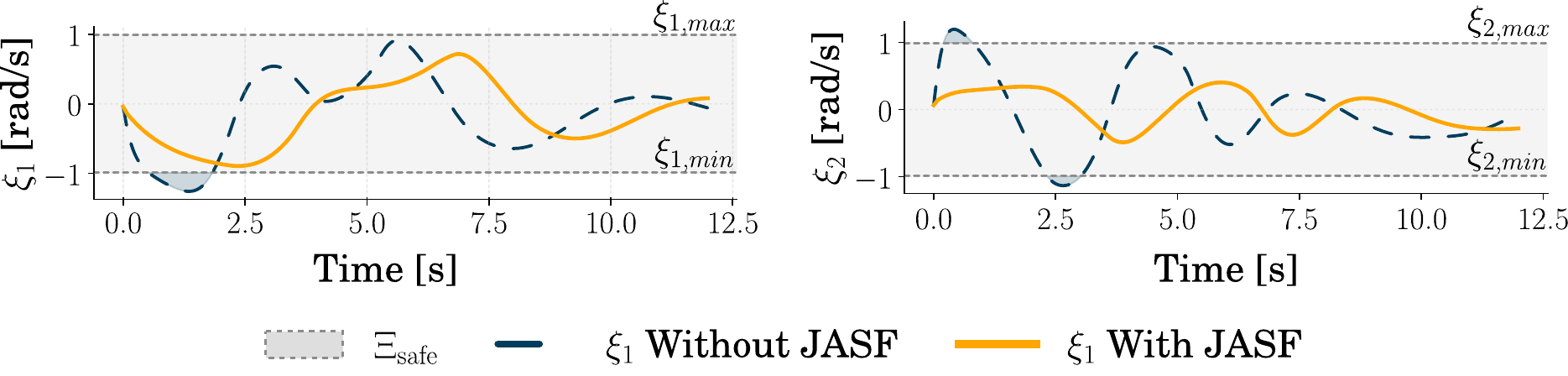}
        \caption{Evolution of commanded joint velocity. Grey region depicts the $\Xi_{safe}:= \{\xi_{min}, \xi_{max}\}$}
        \label{fig:sub3}
    \end{subfigure}
    
    \caption{Experimental comparison between using a JASF with the low-level CBF-QP filter (\textcolor{myorange}{proposed}) and without the JASF layer and only low-level CBF-QP filter (\textcolor{myblue}{baseline}) for a simple $n=2$ planar manipulator in avoiding an obstacle.}
    \label{fig:qual}
\end{figure}
\begin{figure}[htbp] 
\centering  

    \begin{subfigure}[t]{0.32\linewidth} 
        \centering 
        \includegraphics[width=\linewidth]{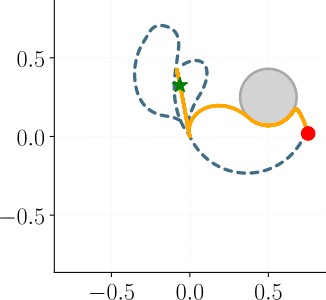} 
        \caption{Trial 1} 
        \label{fig:qual_sweep_a} 
    \end{subfigure} 
    \hfill 
    \begin{subfigure}[t]{0.32\linewidth} 
        \centering 
        \includegraphics[width=\linewidth]{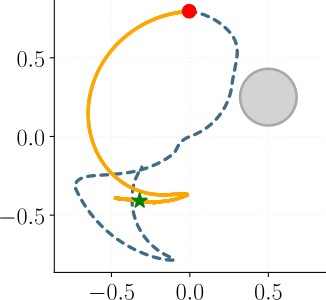} 
        \caption{Trial 2} 
        \label{fig:qual_sweep_b} 
    \end{subfigure} 
    \hfill 
    \begin{subfigure}[t]{0.32\linewidth} 
        \centering 
        \includegraphics[width=\linewidth]{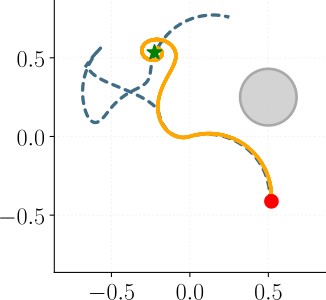} 
        \caption{Trial 3} 
        \label{fig:qual_sweep_c} 
    \end{subfigure} 

\vspace{0.5em}  

    \begin{subfigure}[t]{0.32\linewidth} 
    \centering 
        \includegraphics[width=\linewidth]{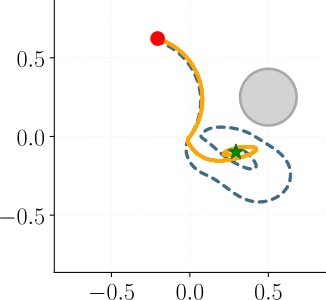} 
        \caption{Trial 4} 
        \label{fig:qual_sweep_d} 
    \end{subfigure} 
    \hfill 
    \begin{subfigure}[t]{0.32\linewidth} 
        \centering 
        \includegraphics[width=\linewidth]{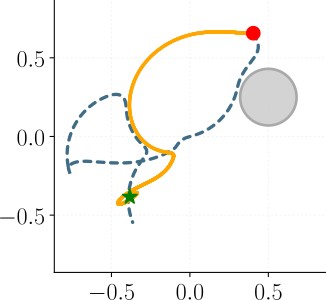} 
        \caption{Trial 5} 
        \label{fig:qual_sweep_e} 
    \end{subfigure} 
    \hfill 
    \begin{subfigure}[t]{0.32\linewidth} 
        \centering 
        \includegraphics[width=\linewidth]{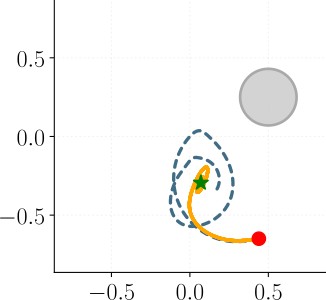} 
        \caption{Trial 6} 
        \label{fig:qual_sweep_f} 
    \end{subfigure} 

\vspace{0.95em} 
\centerline{\includegraphics[width=\linewidth]{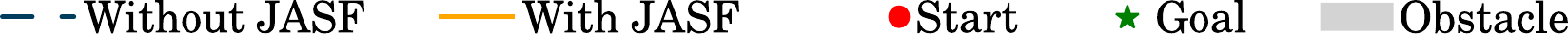}}

\caption{Qualitative comparison in trajectories of 6 random trials for a 2R planar manipulator.} 
\label{fig:qual_sweep} 
\end{figure}

\subsubsection{Quantitative Analysis}\label{sec:exp_2r_manip_quant}
To quantify the performance gains, the proposed JASF method is compared against baseline vanilla CBF method over 50 randomized trials. In each trial, the start and goal poses are chosen randomly. Figure~\ref{fig:quant} shows plots comparing four critical metrics, a) the max $\xi$, b) the number of violation steps of VE in each trial, c) the min$(h(q))$, and d) the final distance from goal.
Table~\ref{tab:quant} shows the mean and variance of the metrics for analytical metrics. While both the methods ensure that the low-level constraints $h(q)$ are always met, the JASF reduced the number of VE violations significantly. Additionally, it lead to a more conservative bound on $h(q)$ which means that the controller was pushed closer to the constraints on fewer occasions. The JASF-based approach also outperformed the baseline in goal reaching behavior and controlling the demanded nominal velocity $\xi$.
\begin{figure}[htbp] 
\centering  

\begin{subfigure}[t]{0.48\linewidth} 
\centering 
\includegraphics[width=\linewidth]{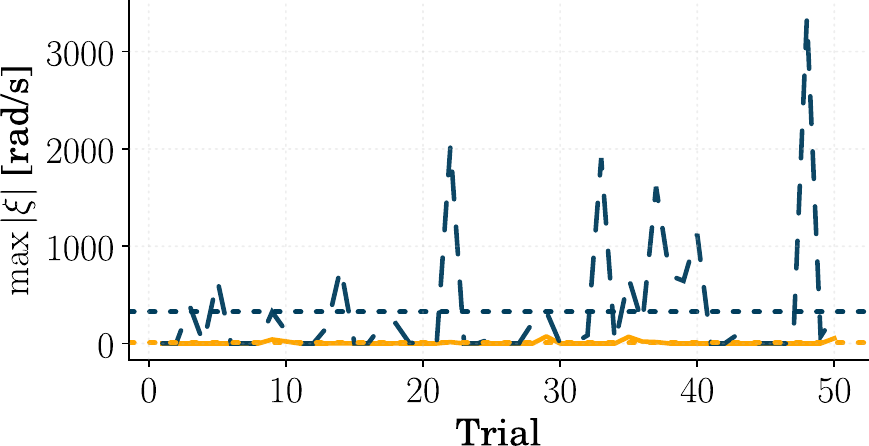} 
\caption{Magnitude of maximum commanded velocity.} 
\label{fig:quant_panel_a} 
\end{subfigure} 
\hfill 
\begin{subfigure}[t]{0.48\linewidth} 
\centering 
\includegraphics[width=\linewidth]{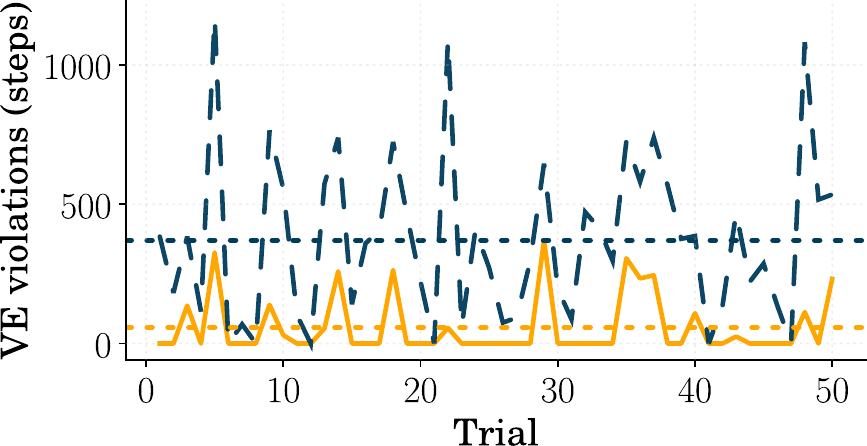} 
\caption{Number of steps with VE violation.} 
\label{fig:quant_panel_b} 
\end{subfigure}  

\vspace{0.5em}  

\begin{subfigure}[t]{0.48\linewidth} 
\centering 
\includegraphics[width=\linewidth]{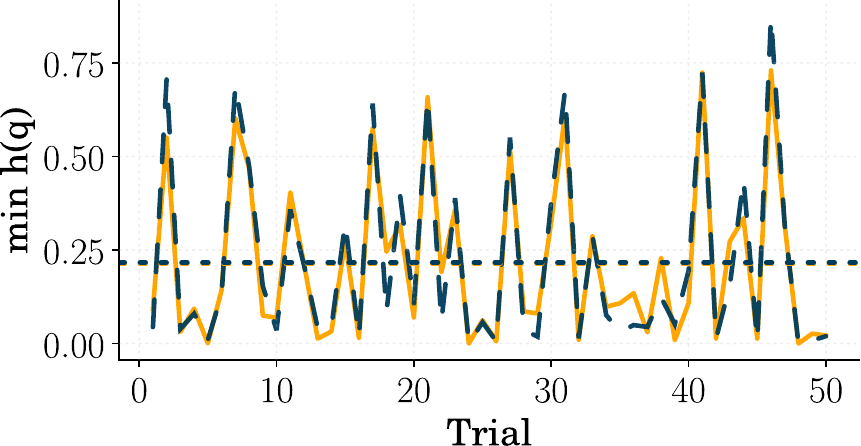} 
\caption{Minimum of $h(q)$.} 
\label{fig:quant_panel_c} 
\end{subfigure} 
\hfill 
\begin{subfigure}[t]{0.48\linewidth} 
\centering 
\includegraphics[width=\linewidth]{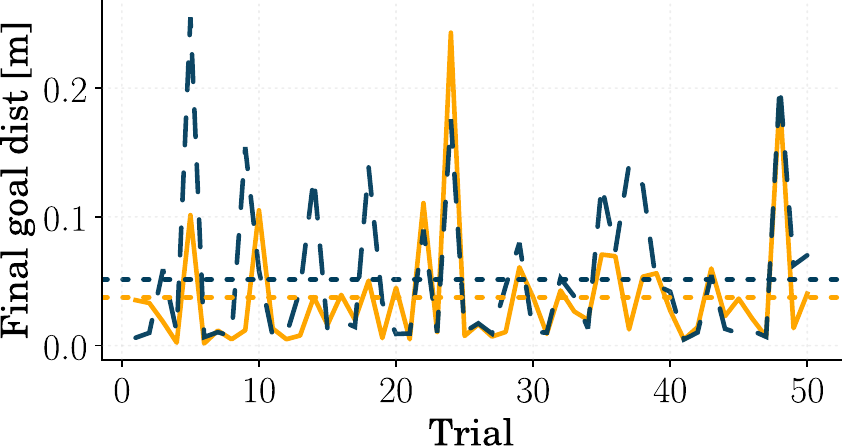} 
\caption{Final distance from goal.} 
\label{fig:quant_panel_d} 
\end{subfigure}  
\caption{Quantitative Analytical comparison after 50 random trials between the baseline without JASF (\textcolor{myblue}{Blue}) and proposed with JASF (\textcolor{myorange}{Orange}). Dotted horizontal line show mean values.} 
\label{fig:quant} 
\end{figure}

\begin{table}[htbp]
\centering
\caption{Safety and MCIS Viability Summary With and Without JASF}
\label{tab:safety_summary}
    \begin{tabular}{l|l|c}
        \hline
        \rule{0pt}{2.0ex}
        \textbf{Method} & \textbf{Metric} & \textbf{Value} \\
        \hline
        \rule{0pt}{2.5ex}
        \multirow{3}{*}{\textcolor{myblue}{\textbf{Without JASF}}}
        & Minimum $h(q)$ & $3.91 \times 10^{-4}$ \\
        & Safety Violations & 0 \\
        & VE Violations & 233 \\
        \hline
        \rule{0pt}{2.5ex}
        \multirow{3}{*}{\textcolor{myorange}{\textbf{With JASF}}}
        & Minimum $h(q)$ & $1.316 \times 10^{-3}$ \\
        & Safety Violations & 0 \\
        & VE Violations & 0 \\
        \hline
    \end{tabular}
\end{table}

\begin{table}[htbp]
    \centering
    \caption{Mean data on trajectory analysis after 50 trials. (Bold values highlight the minimum of the two.)}
    \label{tab:quant}
    \begin{tabular}{l|c|c}
        \hline
        \rule{0pt}{2.5ex}
         \textbf{Metric} & \textbf{\textcolor{myblue}{Without JASF}} & \textbf{\textcolor{myorange}{With JASF}}\\
         \hline
         CBF Violations (Steps) & $0.0 \pm 0.000$ & $0.0 \pm 0.000$\\
         VE Violations (Steps) & $395.360 \pm 263.767$ & $\mathbf{53.340 \pm 1.101}$\\
         $\text{min}~ h(q)$ & $\mathbf{0.262 \pm 0.311}$ & $0.278 \pm 0.315$\\
         Goal dist (m) & $0.063 \pm 0.075$ & $\mathbf{0.043 \pm 0.045}$ \\
         $\text{max}~\xi~ \text{(rad/s)}$ & $438.417 \pm 865.649$ & $\mathbf{15.204 \pm 47.836}$\\
         \hline
    \end{tabular}
\end{table}

\subsection{Comparative Analysis}
\subsubsection{Performance Comparison}
To better understand the benefits of the proposed method, it was compared against two more benchmark methods, a reference governor method and a NMPC-based PSF method. The reference governor follows the same PD-planner based goal-seeking planner and ensures that the reference trajectory always lies within $\Xi_{safe}$ by updating the $u$. Similarly, the NMPC-PSF generates a $N=10$ step look-ahead optimization for ensuring recursive feasibility. Figure~\ref{fig:4way_comp} shows the comparison after 20 random trials. The comparison is based on a) fraction of violations of VE, b) fraction of steps for which either of the inputs was saturated, c) min $h(q)$ which is essentially a measure of risk, and d) total compute time. Table~\ref{tab:4w_comp_perform} lists the values of the performance metrics for the same experiment along with the percentage cases where goal was reached successfully. The proposed JASF does not just significantly reduce the cost of computation but it also shows significant success in goal reaching and reducing VE violations. The JASF method shows significant reduction in VE violation, input saturation, has better goal reaching behavior, and reduced computation compared to NMPC-PSF which is the closest method on other metrics. 
\begin{figure}[htbp] 
\centering  
    \begin{subfigure}[t]{0.48\linewidth} 
        \centering 
        \includegraphics[width=\linewidth]{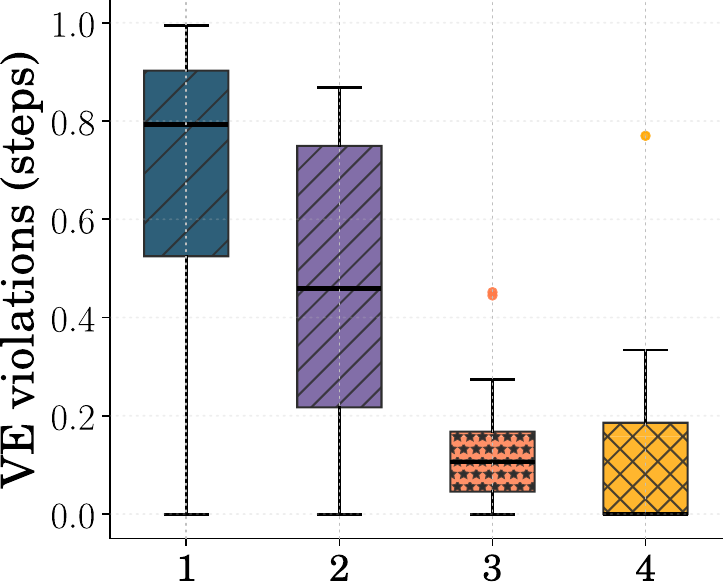} 
        \caption{Fraction of total step with violation of VE.} 
        \label{fig:4way_comp_panel-a} 
    \end{subfigure} 
    \hfill 
        \begin{subfigure}[t]{0.48\linewidth} 
        \centering 
        \includegraphics[width=\linewidth]{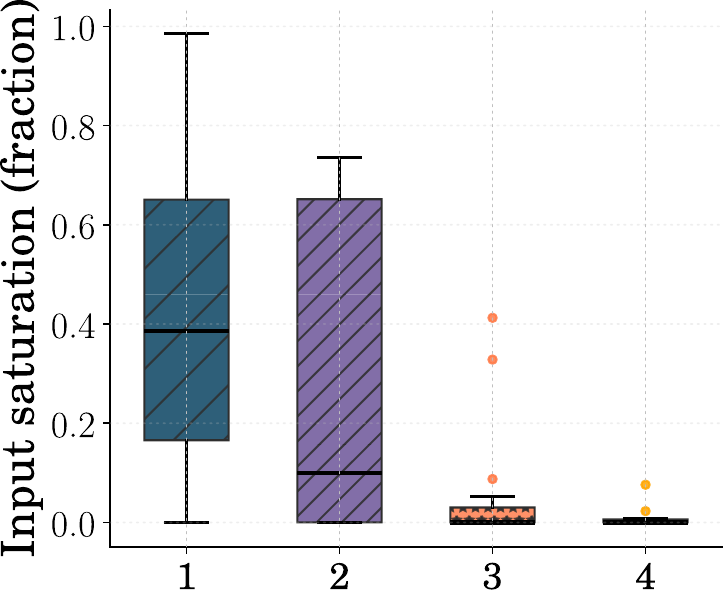} 
        \caption{Fraction of total steps with input saturation.} 
        \label{fig:4way_comp_panel-b} 
    \end{subfigure}  
\vspace{0.5em}  
    \begin{subfigure}[t]{0.48\linewidth} 
        \centering 
        \includegraphics[width=\linewidth]{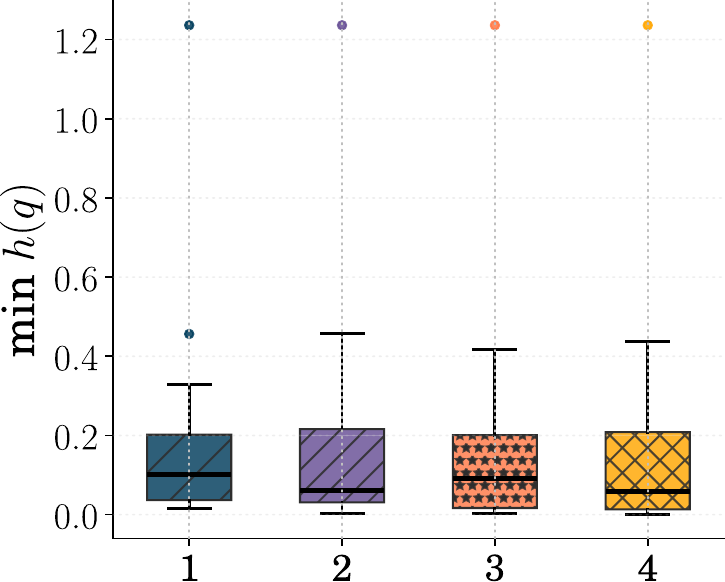} 
        \caption{Minimum of $h(q)$.} 
        \label{fig:4way_comp_panel-c} 
    \end{subfigure} 
    \hfill 
        \begin{subfigure}[t]{0.48\linewidth} 
        \centering 
        \includegraphics[width=\linewidth]{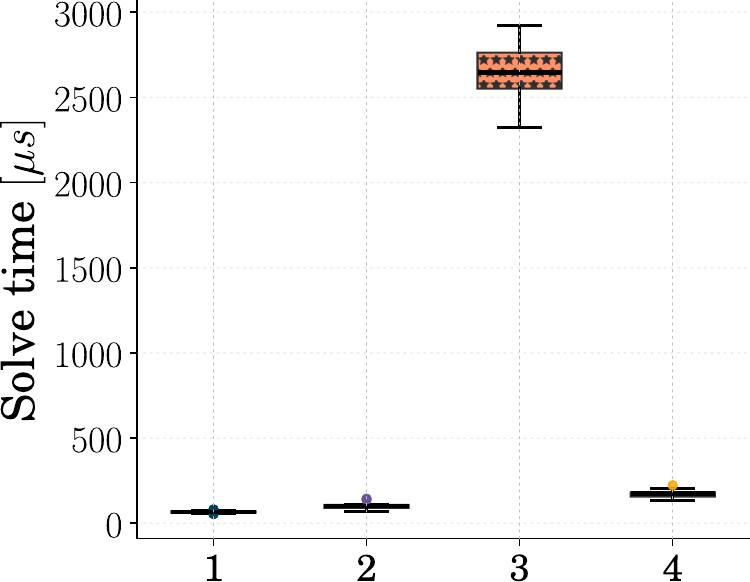} 
        \caption{Solve time.} 
        \label{fig:4way_comp_panel-d} 
    \end{subfigure}  
\vspace{0.95em} 
    \centerline{\includegraphics[width=\linewidth]{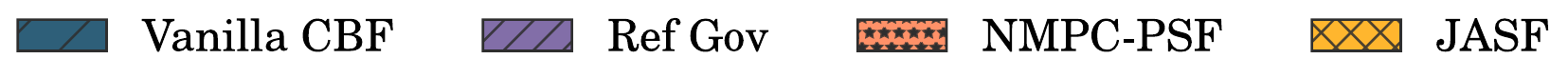}}
    \caption{Performance comparison between different methods.} 
\label{fig:4way_comp} 
\end{figure}

\subsubsection{Benefit in compute time}
Since the NMPC-PSF method was the closest in terms of performance, it essential to thoroughly validate the computational costs. Hence, it is compared against NMPC-PSF approach for the same setup with varying length of horizon. The experiments were repeated for randomized 20 trials. The proposed method achieved substantially lower computation times than NMPC-PSF across all tested prediction horizons, as shown in Fig.\ref{fig:psf}. Specifically, Fig.\ref{fig:psf_comp} indicates nearly an order-of-magnitude reduction in computation time.
\begin{table}[htbp]
\centering
\caption{Benchmark summary over 20 randomized trials. Values are mean $\pm$ standard deviation. (Bold values highlight minimum across all four methods.)}
\label{tab:4w_comp_perform}
\scriptsize
\setlength{\tabcolsep}{5pt}
\begin{tabular}{lcccc}
\toprule
\textbf{Metric} & \textbf{Baseline CBF} & \textbf{Ref.\ Gov} & \textbf{NMPC-PSF} & \textbf{JASF} \\
\midrule
VE frac.      & $0.7{\pm}0.3$   & $0.5{\pm}0.3$   & $0.1{\pm}0.1$   & $\mathbf{0.1{\pm}0.2}$ \\
Sat.\ frac.   & $0.4{\pm}0.3$   & $0.3{\pm}0.3$   & $0.1{\pm}0.1$   & $\mathbf{0.0{\pm}0.0}$ \\
min $h(q)$    & $0.2{\pm}0.3$   & $0.2{\pm}0.3$   & $0.2{\pm}0.3$   & $0.2{\pm}0.3$ \\
Goal dist. [m]& $0.0{\pm}0.0$   & $0.2{\pm}0.3$   & $0.0{\pm}0.0$   & $\mathbf{0.0{\pm}0.0}$ \\
Reached [\%]  & $65.0$          & $70.0$          & $85.0$          & $\mathbf{100.0}$ \\
Solve [$\mu$s]& $\mathbf{66.2{\pm}7.0}$ & $100.6{\pm}17.7$ & $2653.3{\pm}155.4$ & $172.8{\pm}22.6$ \\
\bottomrule
\end{tabular}
\end{table}
In addition, Fig.~\ref{fig:psf_rel} shows an approximately $20\times$ reduction relative to $N=5$ and nearly $65\times$ relative to $N=20$.
\begin{figure}[htbp]
    \centering

    \begin{subfigure}[b]{\linewidth}
        \centering
        \includegraphics[width=\linewidth]{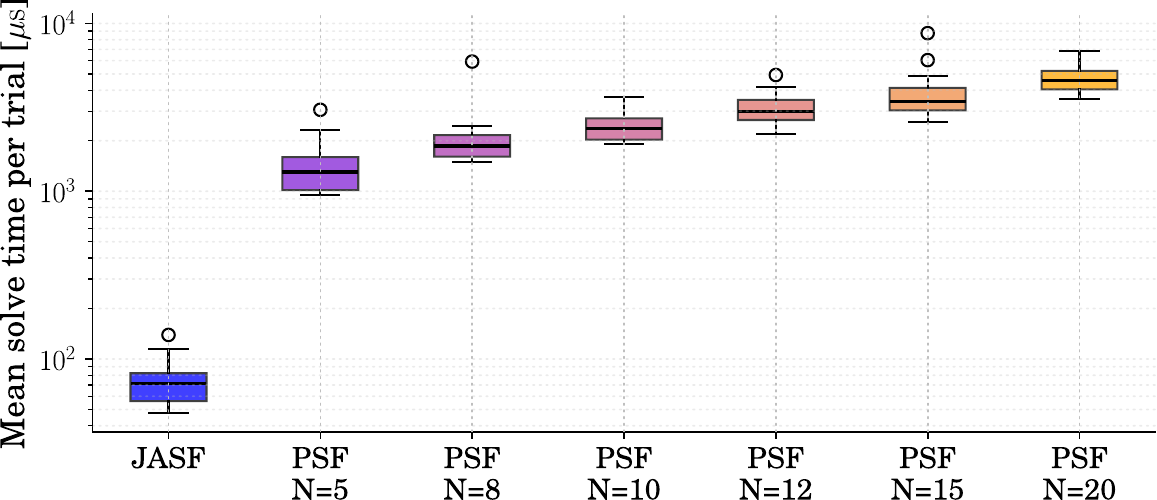}
        \caption{Distribution of mean solve time for per-trial.}
        \label{fig:psf_comp}
    \end{subfigure}
    \vspace{0.5em}
    \begin{subfigure}[b]{\linewidth}
        \centering
        \includegraphics[width=\linewidth]{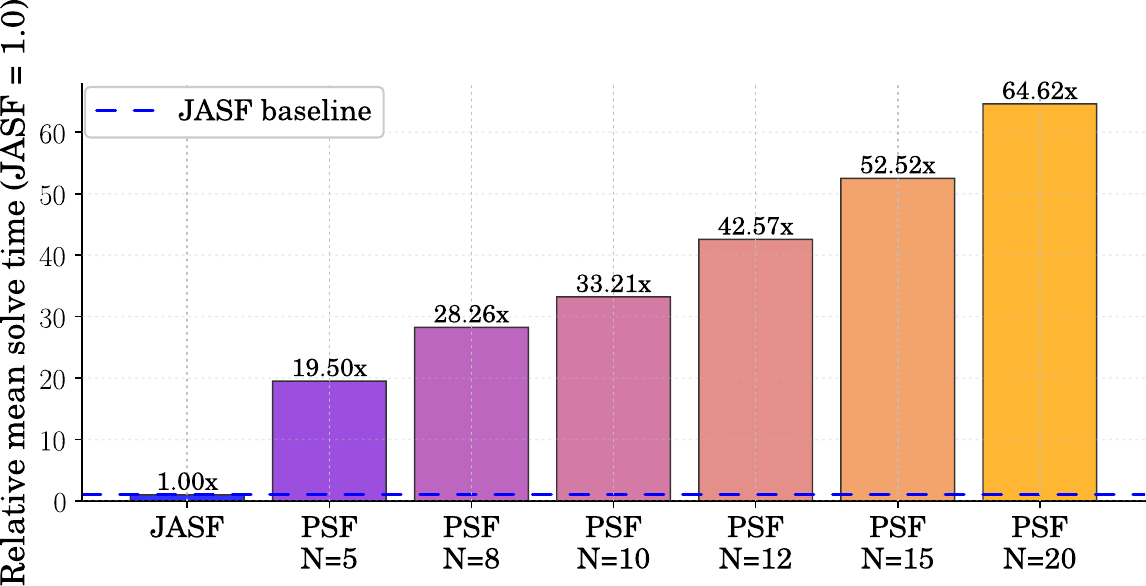}
        \caption{Relative comparison of mean solve time for the different PSF horizon length w.r.t. JASF.}
        \label{fig:psf_rel}
    \end{subfigure}

    \caption{Comparison in compute time for JASF against NMPC-based PSF with varying horizon length over randomized 20 trials.}
    \label{fig:psf}
\end{figure}

\begin{figure}[htbp]
\centering
\begin{subfigure}[t]{0.48\linewidth} 
\centering 
\includegraphics[width=\linewidth]{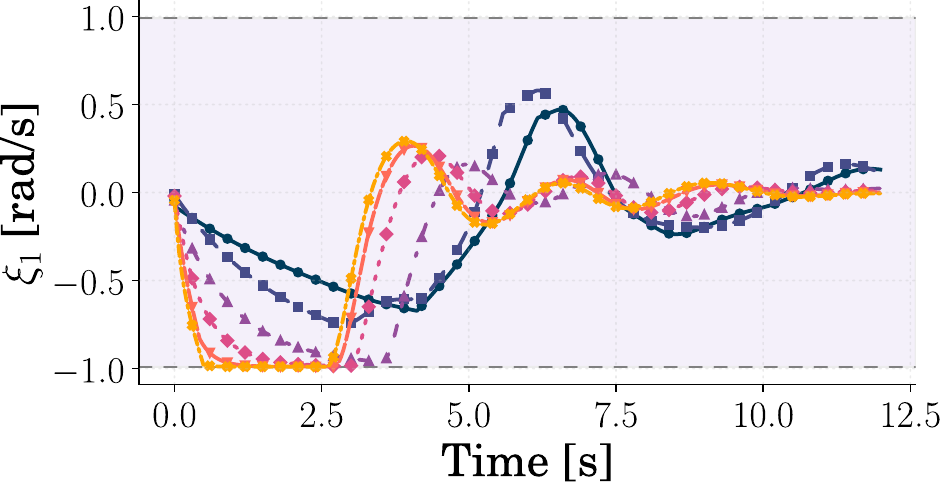} 
\caption{Evolution of $\xi_1$} 
\label{fig:sens_ana_panel_a} 
\end{subfigure} 
\hfill 
\begin{subfigure}[t]{0.48\linewidth} 
\centering 
\includegraphics[width=\linewidth]{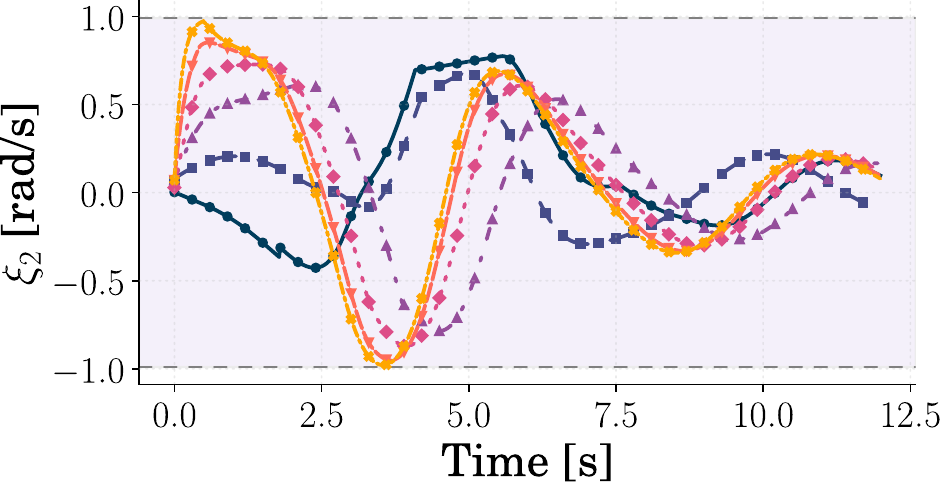} 
\caption{Evolution of $\xi_2$} 
\label{fig:sens_ana_panel_b} 
\end{subfigure}  

\vspace{0.85em} 
\centerline{\includegraphics[width=0.75\linewidth]{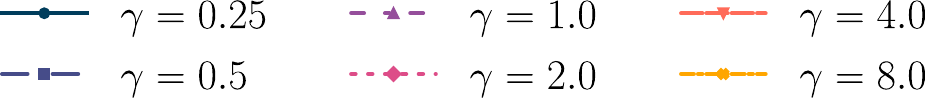}}

\caption{Effect of $\gamma$ on viability.} 
\label{fig:sens_ana} 
\end{figure}

\subsection{Sensitivity Analysis of Conservativeness}
The decay constant $\gamma$, introduced in \eqref{eq:jasf_qp}, determines the decay rate of the JASF-QP. Changing the value of $\gamma$ alters the conservative behavior of the whole system as it controls the decay rate to the safety boundary of $\hat{\mathcal{C}}$. Figure~\ref{fig:sens_ana} highlights the safety analysis for different values of $\gamma$ varying from 0.25 to 8.0. A smaller value of $\gamma$ makes the system more conservative. This allows one to tune the planner and control the conservative behavior of the planner by just adjusting a singular parameter. Enhancing the conservative behavior increases the constraint satisfaction characteristic of the method but can result in reduced performance with regard to reference tracking or goal seeking.

\subsection{Validation on High-DoF Setup}
Figure~\ref{fig:6dof} shows the validation results for a UR5 (6-DoF) robotic arm, simulated using the Robotics Toolbox fo Python\cite{rtb}. A fixed obstacle was placed in the workspace of the robot and a goal pose was provided beyond the obstacle. The experiments highlighted the benefit of having the JASF-QP layer. Without the JASF-QP, the robot gets stuck in a local minima while trying to ensure safety but reaches the limits of joint velocity.  
\begin{figure}[ht]
    \centering
    \includegraphics[width=\linewidth]{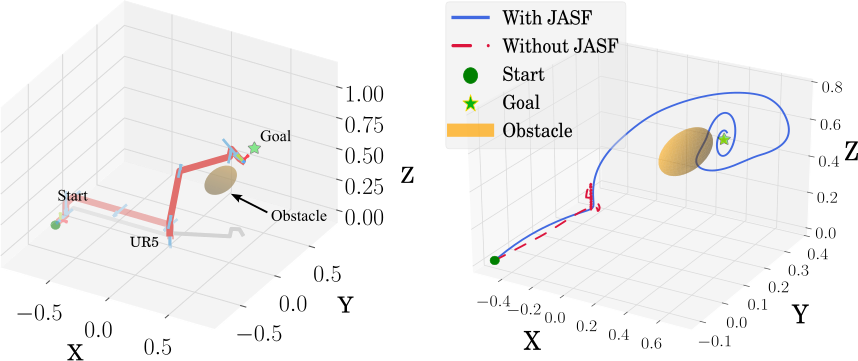}
    \caption{Illustrative demonstration of proposed method on 6-DoF (UR5) robot in simulation.}
    \label{fig:6dof}
\end{figure}
This was avoided by the JASF-QP. This case only focused on ensuring the safety of the end-effector and not the links. After 20 randomized trials, the baseline method resulted in an average final distance from goal of $0.491 \pm 0.152$ m, whereas the proposed JASF method resulted in a significantly better performance with $0.048 \pm 0.022$ m, nearly better by an order of the magnitude. Additionally, the average time for each step was $0.669 \pm 0.20$ milliseconds for the baseline CBF compared to the $1.081 \pm 0.34$ milliseconds for the JASF. This indicates towards potential online implementation for high DoF systems with significant improvement in performance. 



\section{Conclusions}\label{sec:conc}
This paper presented a viability preserving Jacobian-based state-augmented safety filter for kinematically controlled manipulators using a nested-QP architecture. The proposed approach regulates task-space acceleration at a high level so that the induced nominal joint-space command remains within a conservative viability envelope, while a low-level CBF-QP enforces joint-position and joint-velocity safety. In contrast to predictive safety filters that rely on longer-horizon optimization and fuller model knowledge, the proposed method remains model-light and suitable for online implementation. The analysis showed that, under standard regularity assumptions and feasibility of the outer QP, the conservative viability envelope is forward invariant. This provides a tractable mechanism for reducing future infeasibility of the low-level safety filter. Simulation results on a planar 2R manipulator demonstrated that the proposed method reduces viability-envelope violations, limits commanded-velocity growth, and improves goal-reaching behavior relative to a baseline one-step CBF-QP filter. Additional comparisons with a reference governor and an NMPC-based predictive safety filter showed favorable performance-computation tradeoffs, particularly in computation time relative to NMPC-PSF. A simulated 6-DoF example further illustrated applicability to higher-degree-of-freedom manipulators.


\section{Acknowledgments}
This material is based upon work supported by the U.S. Department of Energy Office of Critical Minerals and Energy Innovation, Advanced Materials and Manufacturing Technology Office, under Contract No. DE-AC05-00OR22725 with UT-Battelle, LLC, and conducted at the Oak Ridge National Laboratory (ORNL) Manufacturing Demonstration Facility, a DOE user facility sponsored by the Advanced Materials and Manufacturing Technology Office under CPS Agreement \# 24761.

\bibliographystyle{IEEEtran}
\bibliography{jasf_ref}

@article{Ames_cbf,
  author={Ames, Aaron D. and Xu, Xiangru and Grizzle, Jessy W. and Tabuada, Paulo},
  journal={IEEE Transactions on Automatic Control}, 
  title={Control Barrier Function Based Quadratic Programs for Safety Critical Systems}, 
  year={2017},
  volume={62},
  number={8},
  pages={3861-3876},
  doi={10.1109/TAC.2016.2638961}}

@ARTICLE{kineamtic_cbf,
  author={Singletary, Andrew and Kolathaya, Shishir and Ames, Aaron D.},
  journal={IEEE Control Systems Letters}, 
  title={Safety-Critical Kinematic Control of Robotic Systems}, 
  year={2022},
  volume={6},
  number={},
  pages={139-144},
  doi={10.1109/LCSYS.2021.3050609}}

@INPROCEEDINGS{energy_cbf,
  author={Choudhary, Yogita and Kolathaya, Shishir},
  booktitle={2022 European Control Conference (ECC)}, 
  title={Energy based Control Barrier functions for Robotic Manipulators with Large Safety Constraints}, 
  year={2022},
  volume={},
  number={},
  pages={1328-1335},
  doi={10.23919/ECC55457.2022.9838184}}

@article{wabersich2021predictive,
  author  = {Wabersich, Kim P. and Zeilinger, Melanie N.},
  title   = {A predictive safety filter for learning-based control of constrained nonlinear dynamical systems},
  journal = {Automatica},
  volume  = {129},
  pages   = {109597},
  year    = {2021},
  doi     = {10.1016/j.automatica.2021.109597}
}

@ARTICLE{psf_cbf,
  author={Wabersich, Kim P. and Zeilinger, Melanie N.},
  journal={IEEE Transactions on Automatic Control}, 
  title={Predictive Control Barrier Functions: Enhanced Safety Mechanisms for Learning-Based Control}, 
  year={2023},
  volume={68},
  number={5},
  pages={2638-2651},
  doi={10.1109/TAC.2022.3175628}}

@INPROCEEDINGS{psf_cbf3,
  author={Bena, Ryan M. and Bahati, Gilbert and Werner, Blake and Cosner, Ryan K. and Yang, Lizhi and Ames, Aaron D.},
  booktitle={2025 IEEE-RAS 24th International Conference on Humanoid Robots (Humanoids)}, 
  title={Geometry-Aware Predictive Safety Filters on Humanoids: From Poisson Safety Functions to CBF Constrained MPC}, 
  year={2025},
  volume={},
  number={},
  pages={1-8},
  doi={10.1109/Humanoids65713.2025.11203169}}

@INPROCEEDINGS{psf_cbf2,
  author={Zeng, Jun and Li, Zhongyu and Sreenath, Koushil},
  booktitle={2021 60th IEEE Conference on Decision and Control (CDC)}, 
  title={Enhancing Feasibility and Safety of Nonlinear Model Predictive Control with Discrete-Time Control Barrier Functions}, 
  year={2021},
  volume={},
  number={},
  pages={6137-6144},
  doi={10.1109/CDC45484.2021.9683174}}

@inproceedings{leeman2023predictive,
  author    = {Leeman, Antoine P. and K{\"o}hler, Johannes and Zeilinger, Melanie N.},
  title     = {Predictive safety filter using system level synthesis},
  booktitle = {Proceedings of Machine Learning Research (L4DC)},
  volume    = {211},
  pages     = {1--13},
  year      = {2023}
}

@article{leeman2024stability,
  author  = {Leeman, Antoine P. and others},
  title   = {Stability Mechanisms for Predictive Safety Filters},
  journal = {arXiv preprint arXiv:2404.05496},
  year    = {2024}
}

@inproceedings{choi2021robust,
  author    = {Choi, Jason J. and Lee, Donggun and Sreenath, Koushil and Tomlin, Claire J. and Herbert, Sylvia L.},
  title     = {Robust Control Barrier-Value Functions for Safety-Critical Control},
  booktitle = {IEEE Conference on Decision and Control (CDC)},
  year      = {2021}
}

@article{choi2023safety,
  author  = {Choi, Jason J. and others},
  title   = {On Safety and Liveness Filtering Using Hamilton-Jacobi Reachability Analysis},
  journal = {arXiv preprint arXiv:2312.15347},
  year    = {2023}
}

@misc{zhang2025viability,
  author = {Zhang, Zizhe and Wang, Yicong and Zhang, Zhiquan and Li, Tianyu and Figueroa, Nadia},
  title  = {Viability-Preserving Passive Torque Control (VPP-TC)},
  note   = {To appear in IEEE International Conference on Robotics and Automation (ICRA)},
  year   = {2026},
  eprint = {2510.03367},
  archivePrefix = {arXiv}
}

@inproceedings{roemer2026path,
          title={From Demonstrations to Safe Deployment: Path-Consistent Safety Filtering for Diffusion Policies},
          author={Ralf R{\"o}mer and Julian Balletshofer and Jakob Thumm and Marco Pavone and Angela P. Schoellig and Matthias Althoff},
          booktitle={IEEE International Conference on Robotics and Automation (ICRA)},
          year={2026}
        }

@misc{ubc2012viability,
  author = {UBC},
  title  = {Scalable computation of viability kernels and a viability-theoretic approach},
  school = {University of British Columbia},
  year   = {2012}
}

@inproceedings{gilbert1985fast,
  author    = {Gilbert, Elmer G. and Kolmanovsky, Ilya},
  title     = {Fast Reference Governors for Systems with State and Control Constraints},
  booktitle = {Proceedings of the American Control Conference},
  year      = {1985}
}

@INPROCEEDINGS{governor_kod,
  author={Arslan, Omur and Koditschek, Daniel E.},
  booktitle={2017 IEEE International Conference on Robotics and Automation (ICRA)}, 
  title={Smooth extensions of feedback motion planners via reference governors}, 
  year={2017},
  volume={},
  number={},
  pages={4414-4421},
  doi={10.1109/ICRA.2017.7989510}}

@article{li2023governor,
  author  = {Li, Zhichao and Atanasov, Nikolay},
  title   = {Governor-Parameterized Barrier Function for Safe Output Tracking with Locally Sensed Constraints},
  journal = {Automatica},
  year    = {2023},
  doi     = {10.1016/j.automatica.2023.111003}
}

@article{annurev_safetfilters_haimin,
   author = "Hsu, Kai-Chieh and Hu, Haimin and Fisac, Jaime F.",
   title = "The Safety Filter: A Unified View of Safety-Critical Control in Autonomous Systems", 
   journal= "Annual Review of Control, Robotics, and Autonomous Systems",
   year = "2024",
   volume = "7",
   number = "Volume 7, 2024",
   pages = "47-72",
   doi = "https://doi.org/10.1146/annurev-control-071723-102940",
   url = "https://www.annualreviews.org/content/journals/10.1146/annurev-control-071723-102940",
   publisher = "Annual Reviews",
   issn = "2573-5144",
   type = "Journal Article",
}

@ARTICLE{review_datadriven_sf,
  author={Wabersich, Kim P. and Taylor, Andrew J. and Choi, Jason J. and Sreenath, Koushil and Tomlin, Claire J. and Ames, Aaron D. and Zeilinger, Melanie N.},
  journal={IEEE Control Systems Magazine}, 
  title={Data-Driven Safety Filters: Hamilton-Jacobi Reachability, Control Barrier Functions, and Predictive Methods for Uncertain Systems}, 
  year={2023},
  volume={43},
  number={5},
  pages={137-177},
  doi={10.1109/MCS.2023.3291885}
}

@book{lavalle,
author = {LaValle, Steven M.},
title = {Planning Algorithms},
year = {2006},
isbn = {0521862051},
publisher = {Cambridge University Press},
address = {USA}
}

@inproceedings{rtb,
  title={Not your grandmother’s toolbox--the Robotics Toolbox reinvented for Python},
  author={Corke, Peter and Haviland, Jesse},
  booktitle={2021 IEEE International Conference on Robotics and Automation (ICRA)},
  pages={11357--11363},
  year={2021},
  organization={IEEE}
}

\end{document}